\documentclass{article}
\usepackage{arxiv}
\usepackage{booktabs}
\usepackage{natbib}
\usepackage{tabularx}
\usepackage{amsthm}
\usepackage{amssymb}

\usepackage{multirow}
\usepackage{amsmath}
\usepackage{graphicx}
\usepackage{float}   % for [H] placement specifier
\usepackage[ruled,vlined]{algorithm2e}

\theoremstyle{plain}
\newtheorem{proposition}{Proposition}[section]

\theoremstyle{definition}
\newtheorem{definition}{Definition}[section]

\newcommand{\sg}{\operatorname{sg}}

\title{When Does Margin Clamping Affect Training Variance?\\
Dataset-Dependent Effects in Contrastive Forward-Forward Learning}
\author{
  Joshua Steier\thanks{Correspondence: joshsteier@gmail.com}\\
  Independent Researcher
}
\date{}

\begin{document}
\maketitle

% ============================================================================
% ABSTRACT
% ============================================================================
\begin{abstract}
Contrastive Forward-Forward (CFF) learning trains Vision Transformers
layer by layer against supervised contrastive objectives. CFF training
can be sensitive to random seed, but the sources of this instability
are poorly understood. We focus on one implementation detail: the
positive-pair margin in the contrastive loss is applied through
saturating similarity clamping, $\min(s + m,\, 1)$. We prove that an
alternative formulation, subtracting the margin after the
log-probability, is gradient-neutral under the mean-over-positives
reduction. On CIFAR-10 ($2 \times 2$ factorial, $n{=}7$ seeds per
cell), clamping produces $5.90\times$ higher pooled test-accuracy
variance ($p{=}0.003$) with no difference in mean accuracy. Analyses
of clamp activation rates, layerwise gradient norms, and a
reduced-margin probe point to saturation-driven gradient truncation at
early layers. The effect does not transfer cleanly to other datasets:
on CIFAR-100, SVHN, and Fashion-MNIST, clamping produces equal or
lower variance. Two factors account for the discrepancy. First,
positive-pair density per batch controls how often saturation occurs.
Second, task difficulty compresses seed-to-seed spread when accuracy is
high. An SVHN difficulty sweep confirms the interaction on a single
dataset, with the variance ratio moving from $0.25\times$ at high
accuracy to $16.73\times$ under aggressive augmentation. In
moderate-accuracy regimes with many same-class pairs per batch,
switching to the gradient-neutral subtraction reference removes this
variance inflation at no cost to mean accuracy. Measuring the layer-0
clamp activation rate serves as a simple check for whether the problem
applies.
\end{abstract}

% ============================================================================
% 1. INTRODUCTION
% ============================================================================
\section{Introduction}
\label{sec:intro}

How many seeds do you need to trust a result? In standard
backpropagation-trained models, this question has drawn sustained
interest~\citep{dodge2020finetuning, bhojanapalli2021robustness,
picard2021seed}, but for layer-local training methods the answer
remains open. This paper asks whether common implementation choices
inside the loss function can systematically shift seed-to-seed
variance in one such method.

Forward-Forward (FF) learning~\citep{hinton2022forward} replaces
end-to-end backpropagation with layer-local objectives. Contrastive
Forward-Forward (CFF) extends this idea to Vision Transformers by
training each layer against a supervised contrastive
loss~\citep{aghagolzadeh2025cff, chen2025scff}. Because layers
optimize independently, CFF has qualitatively different optimization
dynamics from standard training. Prior work has focused on mean
accuracy and treated the loss formulation as given. Whether
implementation details inside that loss affect training variance has
not been tested.

\paragraph{Focus.}
We examine one such detail: how the positive-pair margin is applied.
The default in existing CFF codebases uses saturating similarity
clamping, $\min(s + m,\, 1)$. We compare this against a
post-log-probability subtraction alternative and prove that the
subtraction form is gradient-neutral under the mean-over-positives
reduction (Proposition~\ref{prop:logprob_margin_constant}), so it
serves as a true no-margin baseline. The comparison therefore
separates the effect of saturation from the effect of the margin
itself.

\paragraph{Main results.}
On CIFAR-10 ($2\times 2$ factorial, $n{=}7$ seeds per cell),
clamping produces $5.90\times$ higher variance in test accuracy
($p{=}0.003$) with no detectable mean difference. Gradient norms
and clamp activation rates point to truncation at early layers
where saturation is most frequent. The effect is
dataset-dependent: CIFAR-100, SVHN, and Fashion-MNIST all show
inverted variance ratios. Cross-dataset analysis and an SVHN
difficulty sweep identify two moderating factors: positive-pair
density (which governs how often saturation occurs) and task
difficulty (which limits how far seeds can diverge).

\paragraph{Contributions.}
We give a closed-form specification of the CFF supervised
contrastive loss with explicit margin variants
(Section~\ref{sec:methods}), including a proof of
gradient-neutrality for the subtraction form. We then present
an empirical variance audit on CIFAR-10 with supporting
diagnostics (Sections~\ref{sec:experiments}
and~\ref{sec:mechanism}), and map out when the effect does
and does not appear across four datasets
(Section~\ref{sec:generalization}).

% ============================================================================
% 2. BACKGROUND / RELATED WORK
% ============================================================================
\section{Background and Related Work}
\label{sec:background}

\subsection{Forward-Forward and layer-local learning}

Forward-Forward (FF) learning~\citep{hinton2022forward} trains networks
using local objectives rather than a single global loss. It belongs to a
broader family of greedy layerwise methods that optimize each layer (or
block) against a local target~\citep{belilovsky2019greedy,
nokland2019training, lowe2019putting}. CFF extends FF to Vision
Transformers by applying a supervised contrastive objective at each
layer~\citep{aghagolzadeh2025cff, chen2025scff}. Greedy layerwise
methods have been studied for their representational properties and
scalability, but their seed-to-seed stability has seen comparatively
little scrutiny.

\subsection{Contrastive and supervised contrastive learning}

Self-supervised contrastive objectives learn representations by pulling
positive pairs together and pushing negatives apart in a
temperature-scaled softmax~\citep{oord2018infonce, chen2020simclr,
he2020moco}. \citet{khosla2020supcon} extended this to the supervised
setting, where all same-class pairs are treated as positives. The loss
used in CFF follows this supervised formulation with the ``mean over
positives'' reduction. Variants of the supervised contrastive loss
differ in their reduction over positive pairs and in how margins or
similarity modifications enter the
objective~\citep{graf2021dissecting}; our work examines one such
modification.

\subsection{Margins in metric and contrastive learning}

Margins are widely used in metric learning to enforce separation between
classes. Additive angular and cosine margins in face recognition
losses, including
SphereFace~\citep{liu2017sphereface},
CosFace~\citep{wang2018cosface}, and
ArcFace~\citep{deng2019arcface}, reshape logit contributions through
nonlinear operations that can create saturated regions with reduced or
truncated gradients. Our setting differs in two respects:
(i)~the objective is a supervised contrastive loss applied independently
at each layer, and (ii)~Forward-Forward training enforces layer-local
updates through stop-gradients. We document how a saturation-inducing
margin behaves in this layer-local regime, with seed-to-seed variance
rather than mean accuracy as the primary endpoint.

\subsection{Training variance and reproducibility}

Seed-to-seed variability in deep learning has been documented across
fine-tuning~\citep{dodge2020finetuning},
pre-training~\citep{bhojanapalli2021robustness}, and broader model
selection~\citep{damour2022underspecification}. \citet{picard2021seed}
showed that seed choice alone can meaningfully shift reported metrics.
Sources of variance in backpropagation-trained models (weight
initialization, data ordering, dropout masks) are well characterized,
but analogous analyses for layer-local training methods are largely
absent. Our work fills this gap for CFF.

\subsection{Gradient truncation in optimization}

Gradient clipping and truncation have been studied both as deliberate
regularizers~\citep{pascanu2013difficulty, zhang2020why} and as
unintended consequences of activation saturation. In our setting,
similarity clamping introduces a form of implicit gradient truncation
that is data-dependent and varies across layers, setting it apart from
explicit gradient clipping strategies.

\section{Problem Setup}
\label{sec:problem_setup}

Let $\sigma$ denote a random seed that jointly determines weight
initialization, data-loader ordering, and stochastic augmentation for a
CFF training run. Fix an experimental condition $c$ and let
$a_c(\sigma)$ denote the resulting test accuracy on a held-out set after
the full two-stage pipeline (Section~\ref{sec:exp_setup}).

\paragraph{Primary endpoint.}
Given $n$ independent seeds $\{\sigma_1, \dots, \sigma_n\}$, we measure
the sample variance
\begin{equation}
\widehat{\mathrm{Var}}_c
= \frac{1}{n-1} \sum_{k=1}^{n} \bigl(a_c(\sigma_k) - \bar{a}_c\bigr)^2,
\label{eq:sample_var}
\end{equation}
where $\bar{a}_c$ is the sample mean. We test the null hypothesis of
equal population variances between margin types:
\begin{equation}
H_0\colon \mathrm{Var}(\texttt{clamp})
= \mathrm{Var}(\texttt{subtract}).
\label{eq:null_hyp}
\end{equation}

\paragraph{Experimental factors.}
We cross two factors in a $2 \times 2$ factorial design:
(A)~margin type (\texttt{clamp} vs.\ \texttt{subtract};
Section~\ref{sec:margin_variants}), and (B)~numerical-stability mode
(\texttt{detach} vs.\ \texttt{direct};
Eq.~\eqref{eq:stability_modes}). All other pipeline components
(architecture, optimizer, learning rate, augmentation distribution,
and training duration) are held fixed across cells.

\paragraph{Diagnostic endpoints.}
To probe the mechanism behind any observed variance difference, we
additionally report: (i)~clamp activation rates (CAR), which quantify
how often margin-induced saturation occurs per layer; (ii)~layerwise
gradient $\ell_2$ norms; and (iii)~a reduced-margin dose-response probe
that tests whether lowering the margin schedule weakens the variance
effect.
% ============================================================================
\section{Methods}
\label{sec:methods}

\subsection{Notation and representations}
\label{sec:notation}

Table~\ref{tab:notation} summarizes the main symbols; all definitions
are scoped to a single layer $\ell$ unless otherwise stated.

\begin{table}[t]
\centering
\small
\setlength{\tabcolsep}{6pt}
\renewcommand{\arraystretch}{1.15}
\begin{tabularx}{\linewidth}{@{}p{3.0cm}X@{}}
\toprule
\textbf{Symbol} & \textbf{Definition} \\
\midrule
$\mathcal{D}$ &
Labeled dataset $\{(x_n,y_n)\}_{n=1}^{N}$ with
$y_n\in\{1,\dots,K\}$. \\

$B$ &
Minibatch size (number of images per iteration). \\

$x_i^{(1)}, x_i^{(2)}$ &
Two stochastic augmentations of image $x_i$. \\

$L$, $\ell$ &
Number of transformer layers; layer index
$\ell\in\{0,\dots,L{-}1\}$. \\

$z_{i,\ell}^{(v)} \in \mathbb{R}^d$ &
$\ell_2$-normalized mean-pooled representation of view $v$ at
layer $\ell$ (Eq.~\eqref{eq:representation}). \\

$u, v$ &
Indices over the concatenated view set $\{1,\dots,2B\}$
(Definition~\ref{def:concat_index}). \\

$s_{uv,\ell}$ &
Cosine similarity: $z_{u,\ell}^\top z_{v,\ell}\in[-1,1]$. \\

$\tau$ &
Temperature parameter ($\tau > 0$). \\

$m_\ell$ &
Nonnegative margin at layer $\ell$. \\

$M_{uv}$ &
Positive mask: $M_{uv}=1$ iff $y(u)=y(v)$ and $u\neq v$. \\

$\mathcal{P}_u$ &
Positive set: $\{v : M_{uv}=1\}$. \\
\bottomrule
\end{tabularx}
\caption{Core notation. Diagnostic metrics (CAR, gradient norms, VR)
are defined in Section~\ref{sec:metrics}.}
\label{tab:notation}
\end{table}

\paragraph{Representation.}
Let $h_{i,\ell}^{(v)} \in \mathbb{R}^{T \times d}$ denote the
token-level output of layer $\ell$ for example $i$, view $v$, where $T$
is the number of tokens. The contrastive representation is
\begin{equation}
z_{i,\ell}^{(v)}
= \frac{\bar{h}_{i,\ell}^{(v)}}
       {\|\bar{h}_{i,\ell}^{(v)}\|_2},
\qquad
\bar{h}_{i,\ell}^{(v)}
= \frac{1}{T}\sum_{t=1}^{T} h_{i,\ell,t}^{(v)},
\label{eq:representation}
\end{equation}
so $z_{i,\ell}^{(v)} \in \mathbb{R}^d$ lies on the unit sphere.

\begin{definition}[Concatenated view indexing]
\label{def:concat_index}
Given a minibatch of $B$ images with two views each, we form the matrix
$Z_\ell = [z_{1,\ell}^{(1)}, \dots, z_{B,\ell}^{(1)},
z_{1,\ell}^{(2)}, \dots, z_{B,\ell}^{(2)}]^\top
\in \mathbb{R}^{2B \times d}$
and index its rows by $u \in \{1,\dots,2B\}$, with class label $y(u)$
inherited from the underlying image.
\end{definition}

\begin{definition}[Stop-gradient]
\label{def:stop_grad}
For a scalar or tensor $x$, $\sg(x)$ denotes the identity on the
forward pass and zero on the backward pass (i.e., \texttt{detach} in
autodiff frameworks).
\end{definition}

\subsection{Layer-local supervised contrastive objective}
\label{sec:base_loss}

We write the supervised contrastive loss in closed form to make the
margin variants and their gradient behavior explicit.

\paragraph{Base logits.}
Given modified similarities $\tilde{s}_{uv,\ell}$
(Section~\ref{sec:margin_variants}), define
\begin{equation}
b_{uv,\ell} = \frac{\tilde{s}_{uv,\ell}}{\tau}.
\label{eq:base_logits}
\end{equation}

\paragraph{Numerical-stability shift.}
Let $\alpha_{u,\ell} = \max_{k \neq u} b_{uk,\ell}$. We shift logits
row-wise to prevent overflow:
\begin{equation}
g_{uv,\ell} =
\begin{cases}
b_{uv,\ell} - \sg(\alpha_{u,\ell}), & \text{(\texttt{detach} mode)} \\
b_{uv,\ell} - \alpha_{u,\ell},       & \text{(\texttt{direct} mode)}.
\end{cases}
\label{eq:stability_modes}
\end{equation}
Both modes yield identical forward values; they differ in whether
gradients flow through the $\max$ operator
(Definition~\ref{def:stop_grad}).

\paragraph{Per-anchor log-probabilities.}
For anchor $u$ and index $v \neq u$:
\begin{equation}
\log p_{uv,\ell}
= g_{uv,\ell}
  - \log\!\Bigl(\sum_{k \neq u} \exp(g_{uk,\ell})\Bigr).
\label{eq:logprob_def}
\end{equation}

\paragraph{Supervised contrastive loss.}
Using the mean-over-positives reduction:
\begin{equation}
\mathcal{L}_\ell
= -\frac{1}{2B} \sum_{u=1}^{2B}
   \frac{1}{|\mathcal{P}_u|}
   \sum_{v \in \mathcal{P}_u} \log p_{uv,\ell}.
\label{eq:supcon_loss}
\end{equation}

\paragraph{Forward-Forward locality.}
Each $\mathcal{L}_\ell$ is optimized independently using only the
parameters of layer $\ell$; inter-layer gradients are blocked by
stop-gradient operations, so updates at layer $\ell$ do not propagate
into earlier layers.

\subsection{Positive-pair margin variants}
\label{sec:margin_variants}

We compare two strategies that differ in whether they introduce
similarity saturation.

\subsubsection{Saturating similarity clamping}
\label{sec:clamp_margin}

The clamped variant adds the margin in similarity space and caps the
result at~1:
\begin{equation}
\tilde{s}_{uv,\ell} =
\begin{cases}
\min(s_{uv,\ell} + m_\ell,\, 1), & \text{if } M_{uv} = 1, \\
s_{uv,\ell},                      & \text{if } M_{uv} = 0.
\end{cases}
\label{eq:clamp_margin}
\end{equation}

\paragraph{Temperature-margin coupling.}
Because logits are $\tilde{s}/\tau$
(Eq.~\eqref{eq:base_logits}), a similarity-space margin $m_\ell$
amounts to an effective logit shift of $m_\ell / \tau$ until
saturation. The practical strength of clamping therefore depends on
both $m_\ell$ and $\tau$.

\subsubsection{Gradient-neutral subtraction reference}
\label{sec:subtract_margin}

The subtraction baseline computes log-probabilities from unmodified
similarities ($\tilde{s} = s$) and subtracts the margin after the fact:
\begin{equation}
\log \tilde{p}_{uv,\ell}
= \log p_{uv,\ell} - m_\ell \, M_{uv}.
\label{eq:subtract_margin}
\end{equation}

\begin{proposition}[Post-log-probability subtraction is gradient-neutral]
\label{prop:logprob_margin_constant}
Let $m_\ell \geq 0$ be a fixed scalar (not a function of model
parameters). Under the mean-over-positives loss
\eqref{eq:supcon_loss}, replacing $\log p_{uv,\ell}$ with
$\log \tilde{p}_{uv,\ell}$ from \eqref{eq:subtract_margin} shifts
each per-anchor term by the constant $m_\ell$ and leaves all gradients
with respect to model parameters unchanged.
\end{proposition}

\begin{proof}
For a fixed anchor $u$:
\[
-\frac{1}{|\mathcal{P}_u|}
 \sum_{v \in \mathcal{P}_u} \log \tilde{p}_{uv,\ell}
= -\frac{1}{|\mathcal{P}_u|}
   \sum_{v \in \mathcal{P}_u}
   (\log p_{uv,\ell} - m_\ell)
= -\frac{1}{|\mathcal{P}_u|}
   \sum_{v \in \mathcal{P}_u} \log p_{uv,\ell}
  \;+\; m_\ell.
\]
The term $m_\ell$ is constant with respect to all model parameters
(by assumption), so it vanishes under differentiation.
\end{proof}

\paragraph{Interpretation.}
The subtraction baseline does not reshape gradients; it serves as a
gradient-neutral reference. Clamping
(Eq.~\eqref{eq:clamp_margin}), by contrast, can truncate gradients
for any positive pair where saturation occurs.

\subsection{Diagnostic metrics}
\label{sec:metrics}

\paragraph{Variance ratio (primary endpoint).}
\begin{equation}
\mathrm{VR}
= \frac{\widehat{\mathrm{Var}}(\texttt{clamp})}
       {\widehat{\mathrm{Var}}(\texttt{subtract})},
\label{eq:vr}
\end{equation}
where $\widehat{\mathrm{Var}}$ is the sample variance of test
accuracies across seeds (Eq.~\eqref{eq:sample_var}).

\paragraph{Clamp Activation Rate (CAR).}
Let $\mathcal{P}_\ell = \{(u,v) : M_{uv} = 1\}$ be the set of
positive pairs at layer $\ell$. The CAR is
\begin{equation}
\mathrm{CAR}_\ell
= \frac{1}{|\mathcal{P}_\ell|}
  \sum_{(u,v) \in \mathcal{P}_\ell}
  \mathbb{1}[s_{uv,\ell} + m_\ell > 1].
\label{eq:car}
\end{equation}

\paragraph{Gradient norm.}
For each layer $\ell$ with parameter set $\Theta_\ell$:
\begin{equation}
\|\nabla_\ell\|
= \sqrt{\sum_{p \in \Theta_\ell}
        \bigl\|\nabla_p \mathcal{L}_\ell\bigr\|_2^2},
\label{eq:grad_norm}
\end{equation}
i.e., the $\ell_2$ norm of the full gradient vector of
$\mathcal{L}_\ell$ with respect to all parameters in the $\ell$-th
encoder block.

% ============================================================================
% 5. EXPERIMENTS
% ============================================================================
\section{Experiments}
\label{sec:experiments}

\subsection{Experimental setup}
\label{sec:exp_setup}

\paragraph{Datasets.}
The primary analysis uses CIFAR-10 with the standard 50k/10k train/test
split; we hold out 5k training images for validation, yielding 45k
train, 5k validation, and 10k test images. Generalization experiments
(Section~\ref{sec:generalization}) repeat the variance audit on
CIFAR-100 (100 classes), SVHN (10 classes, street-view house numbers),
and Fashion-MNIST (10 classes, grayscale). Dataset-specific details
appear in Section~\ref{sec:generalization}.

\paragraph{Architecture and hyperparameters.}
Table~\ref{tab:hparams} summarizes the full configuration. We use a
Vision Transformer~\citep{dosovitskiy2021vit} with embedding dimension
$d=128$, 4 attention heads, and $L=8$ encoder layers. Input images
($32 \times 32$) are divided into $4 \times 4$ patches, producing
$T=64$ tokens per image.

\begin{table}[t]
\centering
\caption{\textbf{Hyperparameter configuration.}}
\label{tab:hparams}
\begin{tabular}{ll}
\toprule
\textbf{Parameter} & \textbf{Value} \\
\midrule
Architecture & ViT ($d{=}128$, $H{=}4$, $L{=}8$) \\
Patch size & $4 \times 4$ \\
Batch size & 512 \\
Stage 1 (repr.\ learning) & 600 epochs, AdamW
  ($\eta{=}4{\times}10^{-3}$, $\beta{=}(0.9, 0.999)$,
   wd${=}10^{-4}$) \\
Stage 2 (linear probe) & 50 epochs, AdamW
  ($\eta{=}5{\times}10^{-4}$) \\
Temperature $\tau$ & 0.15 \\
Standard margin schedule & $m_0{=}0.4 \to m_{L-1}{=}0.1$
  (linear across layers) \\
Low margin schedule & $m_0{=}0.2 \to m_{L-1}{=}0.1$ \\
Augmentation & \texttt{RandomCrop}(32, pad=12),
  \texttt{HFlip}, channel normalization \\
\bottomrule
\end{tabular}
\end{table}

\paragraph{Data augmentation (CIFAR-10 and CIFAR-100).}
Each image undergoes \texttt{RandomCrop(32, padding=12)} and
\texttt{RandomHorizontalFlip}, followed by per-channel normalization
($\mu=(0.491, 0.482, 0.447)$, $\sigma=(0.202, 0.199, 0.201)$). No
RandAugment, AutoAugment, or color jitter is applied. All conditions
share the same augmentation pipeline; augmentation draws are controlled
by the random seed.

\paragraph{Data augmentation (SVHN and Fashion-MNIST).}
SVHN uses the same 12-pixel crop padding as CIFAR-10 with
dataset-specific normalization
($\mu=(0.438, 0.444, 0.473)$, $\sigma=(0.198, 0.201, 0.197)$).
Fashion-MNIST images ($28 \times 28$, grayscale) are converted to
3-channel and resized to $32 \times 32$, normalized with
$\mu = \sigma = (0.5, 0.5, 0.5)$, and augmented with
\texttt{RandomCrop(32, padding=4)} and \texttt{RandomHorizontalFlip}.
The smaller crop padding (4 vs.\ 12~pixels) reflects Fashion-MNIST's
lower spatial complexity relative to natural images.

\paragraph{SVHN difficulty sweep augmentation.}
The difficulty sweep (Section~\ref{sec:svhn_sweep}) varies augmentation
intensity while holding the dataset and architecture fixed:
(i)~\emph{Easy} (standard): \texttt{RandomCrop(32, padding=12)} and
\texttt{HFlip};
(ii)~\emph{Medium}: \texttt{RandomCrop(32, padding=4)},
\texttt{RandomRotation(10)}, and
\texttt{ColorJitter(0.2, 0.2, 0.2)};
(iii)~\emph{Hard}: \texttt{RandomCrop(32, padding=6)},
\texttt{HFlip(0.5)}, \texttt{RandomRotation(15)},
\texttt{RandomErasing(p=0.5, scale=(0.1, 0.3))}, and
\texttt{ColorJitter(0.4, 0.4, 0.4, 0.1)}.

\paragraph{Training protocol.}
CFF training proceeds in two stages:
\begin{enumerate}
    \item \emph{Representation learning.} Each of the $L{=}8$ layers
    is trained against its own supervised contrastive loss
    (Eq.~\eqref{eq:supcon_loss}) for 600 epochs.
    \item \emph{Linear probe.} A linear classifier is trained for 50
    epochs on frozen representations. The reported test accuracy is the
    checkpoint with the highest validation accuracy during this stage.
\end{enumerate}

\paragraph{Random seed control.}
Each seed $\sigma$ jointly determines weight initialization, data-loader
shuffling order, and stochastic augmentation draws. All hyperparameters
are held fixed across seeds and conditions.

\paragraph{Margin schedule.}
Margins decrease linearly across layers:
\begin{equation}
m_\ell = m_0 + (m_{L-1} - m_0)\,\frac{\ell}{L-1},
\quad \ell \in \{0,\dots,L{-}1\}.
\label{eq:margin_schedule}
\end{equation}

\paragraph{Factorial design.}
Under the standard margin schedule, we cross two factors:
(A)~margin type (\texttt{clamp} vs.\ \texttt{subtract}) and
(B)~stability mode (\texttt{detach} vs.\ \texttt{direct}), with $n{=}7$
independent seeds per cell (28 runs total).

\paragraph{Dose-response probe.}
To test whether reducing the margin weakens the variance effect, we
run a low-margin schedule ($m_0{=}0.2 \to m_{L-1}{=}0.1$) with 7
seeds each for \texttt{clamp\_detach} and \texttt{clamp\_direct} (14
clamped runs). The subtraction reference reuses the 7
\texttt{subtract\_detach} seeds from the standard-margin factorial; this
is valid because the subtraction baseline is gradient-neutral
(Proposition~\ref{prop:logprob_margin_constant}), so the margin
schedule value does not affect the subtract condition's gradients or
trained model. The resulting pooled comparison is unbalanced ($n{=}14$
clamp vs.\ $n{=}7$ subtract); we conserve compute based on the
factorial's finding that stability mode does not affect variance
($p > 0.82$; Table~\ref{tab:factorial_var_tests}).

\paragraph{Compute.}
All runs were conducted on NVIDIA V100-SXM2-32GB GPUs. Each CIFAR-10
run takes roughly 7 GPU-hours. The full experimental campaign
(CIFAR-10 factorial, dose-response, and generalization datasets) totals
roughly 88 runs and ${\sim}550$ GPU-hours.

\subsection{Results: CIFAR-10}
\label{sec:results}

\subsubsection{Standard-margin factorial results}
\label{sec:std_results}

Table~\ref{tab:c10_factorial} reports test accuracy statistics for each
cell of the $2 \times 2$ factorial under the standard margin schedule.
Per-seed accuracies are listed in Appendix~\ref{app:per_seed_c10}.

\begin{table}[t]
\centering
\caption{\textbf{CIFAR-10 factorial results (standard margin
$0.4 \to 0.1$).} Test accuracy across $n{=}7$ seeds per cell.}
\label{tab:c10_factorial}
\begin{tabular}{llccc}
\toprule
\textbf{Margin type} & \textbf{Stability} & \textbf{Mean (\%)}
  & \textbf{Std (\%)} & \textbf{Var} \\
\midrule
\texttt{clamp}    & \texttt{detach} & 78.52 & 0.927 & 0.8590 \\
\texttt{clamp}    & \texttt{direct} & 78.44 & 1.158 & 1.3408 \\
\midrule
\texttt{subtract} & \texttt{detach} & 78.73 & 0.467 & 0.2178 \\
\texttt{subtract} & \texttt{direct} & 78.30 & 0.224 & 0.0500 \\
\bottomrule
\end{tabular}
\end{table}

\paragraph{Factorial variance tests.}
To assess whether variability depends on margin type, stability mode, or
their interaction, we apply Levene and Brown-Forsythe procedures (two-way
ANOVA on absolute deviations from cell means and medians, respectively).
Table~\ref{tab:factorial_var_tests} reports the results. Neither test
finds evidence that stability mode affects variance (both $p > 0.82$).
The margin-type effect is marginal ($p \approx 0.06$) but directionally
consistent with the pooled comparison below. The interaction is not
supported under Brown-Forsythe ($p = 0.137$); the marginal Levene
interaction ($p = 0.058$) appears driven by the higher variance of the
\texttt{clamp\_direct} cell.

\begin{table}[t]
\centering
\caption{\textbf{Factorial variance sensitivity tests (standard margin).}}
\label{tab:factorial_var_tests}
\begin{tabular}{lccc}
\toprule
\textbf{Test} & \textbf{Margin type} & \textbf{Stability mode}
  & \textbf{Interaction} \\
\midrule
Levene          & 0.061 & 0.824 & 0.058 \\
Brown-Forsythe & 0.058 & 0.856 & 0.137 \\
\bottomrule
\end{tabular}
\end{table}

\paragraph{Pooled primary comparison.}
Because stability mode does not affect variance, we pool across
stability modes for the primary analysis, yielding $n{=}14$ seeds per
margin type (Table~\ref{tab:c10_pooled}).

\begin{table}[t]
\centering
\caption{\textbf{Pooled test accuracy by margin type
(standard margin, $n{=}14$ per group).}}
\label{tab:c10_pooled}
\begin{tabular}{lcccc}
\toprule
\textbf{Condition} & \textbf{n} & \textbf{Mean (\%)}
  & \textbf{Std (\%)} & \textbf{Var} \\
\midrule
Clamp (pooled)    & 14 & 78.48 & 1.008 & 1.0170 \\
Subtract (pooled) & 14 & 78.51 & 0.415 & 0.1724 \\
\bottomrule
\end{tabular}
\end{table}

\paragraph{Mean comparison.}
A Welch two-sample $t$-test finds no evidence of a mean difference
between margin types ($t(19.4) = 0.10$, $p = 0.92$; 95\% CI for
$\mu_{\texttt{clamp}} - \mu_{\texttt{subtract}}$: $[-0.64, 0.59]$).
The variance difference therefore cannot be attributed to a shift in
central tendency.

\paragraph{Variance comparison.}
An F-test for equality of variances yields $F(13, 13) = 5.90$,
$p = 0.003$, rejecting $H_0$ (Eq.~\eqref{eq:null_hyp}). We
additionally report:
\begin{itemize}
    \item Levene's test: $p = 0.058$ (marginal, consistent in
    direction).
    \item Bootstrap 95\% CI for VR: $[1.62, 15.80]$ (10{,}000
    percentile-bootstrap resamples of the ratio of sample variances).
\end{itemize}
Shapiro-Wilk tests do not reject normality for either pooled group
($W = 0.951$, $p = 0.58$ for clamp; $W = 0.952$, $p = 0.59$ for
subtract; Table~\ref{tab:shapiro} in
Appendix~\ref{app:normality}), supporting the F-test assumption.

Figure~\ref{fig:seed_spread} shows the per-seed accuracy
distributions.

\begin{figure}[t]
\centering
\includegraphics[width=0.55\linewidth]{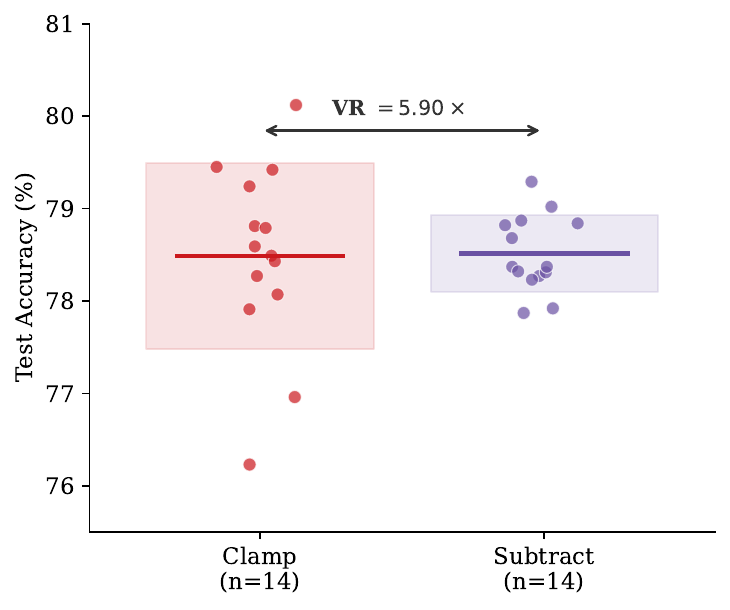}
\caption{\textbf{Per-seed test accuracy by margin type (standard
margin, pooled).} Each dot is one seed ($n{=}14$ per group).
Horizontal lines show group means; shaded bands show
$\pm 1$ standard deviation.}
\label{fig:seed_spread}
\end{figure}

\subsubsection{Dose-response probe}
\label{sec:dose_response}

If the variance effect is mediated by saturation, reducing the starting
margin should weaken it. Table~\ref{tab:dose_response} compares the
two margin schedules.

\begin{table}[t]
\centering
\caption{\textbf{Dose-response probe.} Reducing the margin schedule
from $0.4 \to 0.1$ to $0.2 \to 0.1$ lowers the observed variance
ratio. The low-margin comparison is unbalanced ($n{=}14$ clamp vs.\
$n{=}7$ subtract; see Section~\ref{sec:exp_setup}).}
\label{tab:dose_response}
\begin{tabular}{lcccc}
\toprule
\textbf{Schedule} & \textbf{Clamp var} & \textbf{Subtract var}
  & \textbf{VR} & \textbf{F-test $p$} \\
\midrule
$0.4 \to 0.1$ (standard) & 1.0170 & 0.1724
  & $5.90\times$ & 0.003 \\
$0.2 \to 0.1$ (low) & 0.6498 & 0.2178
  & $2.98\times$ & 0.19 \\
\bottomrule
\end{tabular}
\end{table}

Halving the starting margin from $m_0 = 0.4$ to $m_0 = 0.2$ cut
the variance ratio from $5.90\times$ to $2.98\times$, consistent with
a dose-response relationship. The low-margin ratio does not reach
significance on its own ($F(13,6) = 2.98$, $p = 0.19$; bootstrap
95\% CI for VR: $[0.79, 30.77]$), likely reflecting lower power
from the smaller effect size and unbalanced design ($n{=}14$ vs.\
$n{=}7$). We therefore read this probe as directional evidence
for the saturation-mediated mechanism, not as a standalone
significance claim.

A Welch $t$-test on the low-margin pooled means yields $t(18.4) =
-1.59$, $p = 0.13$ (95\% CI for
$\mu_{\texttt{clamp}} - \mu_{\texttt{subtract}}$: $[-1.03, +0.14]$
pp), showing no significant mean difference, though the trend is
less cleanly null than under the standard margin ($p = 0.92$).

This probe has two limitations. First, because $\tau$ is fixed,
lowering $m_0$ simultaneously reduces the effective logit shift
$m_0 / \tau$, confounding changes in saturation frequency
(Eq.~\eqref{eq:car}) with overall logit-scale changes. Second, the
bootstrap CI for the low-margin VR includes 1.0, so the data are
also consistent with no variance difference under this reduced
schedule.
% ============================================================================
% 6. MECHANISM ANALYSIS
% ============================================================================
\section{Mechanism Analysis}
\label{sec:mechanism}

This section presents diagnostic evidence bearing on why the clamped
variant shows higher seed-to-seed variance on CIFAR-10. We stress
that the diagnostics establish necessary ingredients for the proposed
mechanism but do not amount to a causal proof.

\subsection{Gradient truncation under saturation}
\label{sec:clamp_grad}

The derivative of $f(x) = \min(x, 1)$ is $f'(x) = 1$ for $x < 1$ and
$f'(x) = 0$ for $x > 1$. By Eq.~\eqref{eq:clamp_margin}, for any
positive pair $(u,v)$ at layer $\ell$ with
$s_{uv,\ell} + m_\ell > 1$, the modified similarity
$\tilde{s}_{uv,\ell} = 1$ is constant as a function of
$s_{uv,\ell}$, so the \emph{direct} (numerator) gradient contribution
of that pair to anchor $u$'s loss term is zero.

\paragraph{Denominator pathway.}
Gradient truncation along the direct path does not mean the full
derivative $\partial \mathcal{L}_\ell / \partial s_{uv,\ell}$ is zero.
The clamped similarity $\tilde{s}_{uv,\ell} = 1$ still enters the
softmax denominators of \emph{other} anchors $u' \neq u$ through
Eq.~\eqref{eq:logprob_def}. Gradients can therefore flow through
this indirect (denominator) pathway. The net effect of saturation is
a partial, not complete, gradient truncation, one that is
data-dependent and varies with which pairs happen to exceed the
saturation threshold on each minibatch.

\subsection{Clamp Activation Rate by layer}
\label{sec:car_profile}

Table~\ref{tab:car_full} reports CAR (Eq.~\eqref{eq:car}) across all
8 layers for both margin schedules, computed at the final Stage-1 epoch
(epoch 600) by averaging over minibatches and seeds.

\begin{table}[t]
\centering
\caption{\textbf{Clamp Activation Rate (CAR) by layer (CIFAR-10).}}
\label{tab:car_full}
\begin{tabular}{ccccc}
\toprule
\textbf{Layer} & \textbf{Margin (std / low)}
  & \textbf{CAR (std)} & \textbf{CAR (low)}
  & \textbf{Reduction} \\
\midrule
0 & 0.40 / 0.20 & 60.7\% & 40.3\% & 33.7\% \\
1 & 0.36 / 0.19 & 58.3\% & 43.7\% & 24.9\% \\
2 & 0.31 / 0.17 & 54.0\% & 42.5\% & 21.3\% \\
3 & 0.27 / 0.16 & 49.6\% & 40.5\% & 18.3\% \\
4 & 0.23 / 0.14 & 45.8\% & 38.6\% & 15.6\% \\
5 & 0.19 / 0.13 & 41.6\% & 36.6\% & 12.0\% \\
6 & 0.14 / 0.11 & 37.5\% & 34.5\% &  7.8\% \\
7 & 0.10 / 0.10 & 32.3\% & 32.3\% &  0.2\% \\
\bottomrule
\end{tabular}
\end{table}

Higher starting margins produce markedly higher CAR in early layers,
with the gap narrowing monotonically and vanishing at layer~7 where
both schedules share $m = 0.1$. Under the low-margin schedule, CAR
rises slightly from layer~0 (40.3\%) to layer~1 (43.7\%) despite a
lower margin at layer~1; this likely reflects tighter within-class
cosine similarity at layer~1 (i.e., more clustered representations),
though we do not pursue this further.

\subsection{Gradient norms by layer}
\label{sec:grad_norms}

Table~\ref{tab:grad_norms} reports gradient $\ell_2$ norms
(Eq.~\eqref{eq:grad_norm}) at the final training epoch for layers
0--3. We restrict to these layers because CAR drops below 50\% by
layer~3 under the standard schedule, and layers 4--7 show negligible
differences between conditions ($< 5\%$ relative difference in
gradient norms; not tabulated).

\begin{table}[t]
\centering
\caption{\textbf{Mean gradient norms by layer (final epoch, CIFAR-10).}
Values are mean $\pm$ std across seeds.}
\label{tab:grad_norms}
\begin{tabular}{lccc}
\toprule
\textbf{Layer} & \textbf{Clamp} & \textbf{Subtract}
  & \textbf{Ratio (subtract / clamp)} \\
\midrule
L0 & $0.400 \pm 0.028$ & $1.614 \pm 0.045$ & $4.0\times$ \\
L1 & $0.115 \pm 0.008$ & $0.243 \pm 0.008$ & $2.1\times$ \\
L2 & $0.048 \pm 0.002$ & $0.049 \pm 0.008$ & $1.0\times$ \\
L3 & $0.023 \pm 0.003$ & $0.018 \pm 0.007$ & $0.8\times$ \\
\bottomrule
\end{tabular}
\end{table}

At layer~0, where CAR is highest (60.7\%), clamping reduces gradient
norms by a factor of $4.0\times$ relative to the subtraction baseline.
The reduction narrows to $2.1\times$ at layer~1 and disappears by
layer~2. This correspondence between CAR and gradient norm reduction
across layers is consistent with saturation-induced truncation being
the dominant cause of the gradient difference at early layers.

Figure~\ref{fig:car_gradnorms} plots both profiles together.

\begin{figure}[t]
\centering
\includegraphics[width=\linewidth]{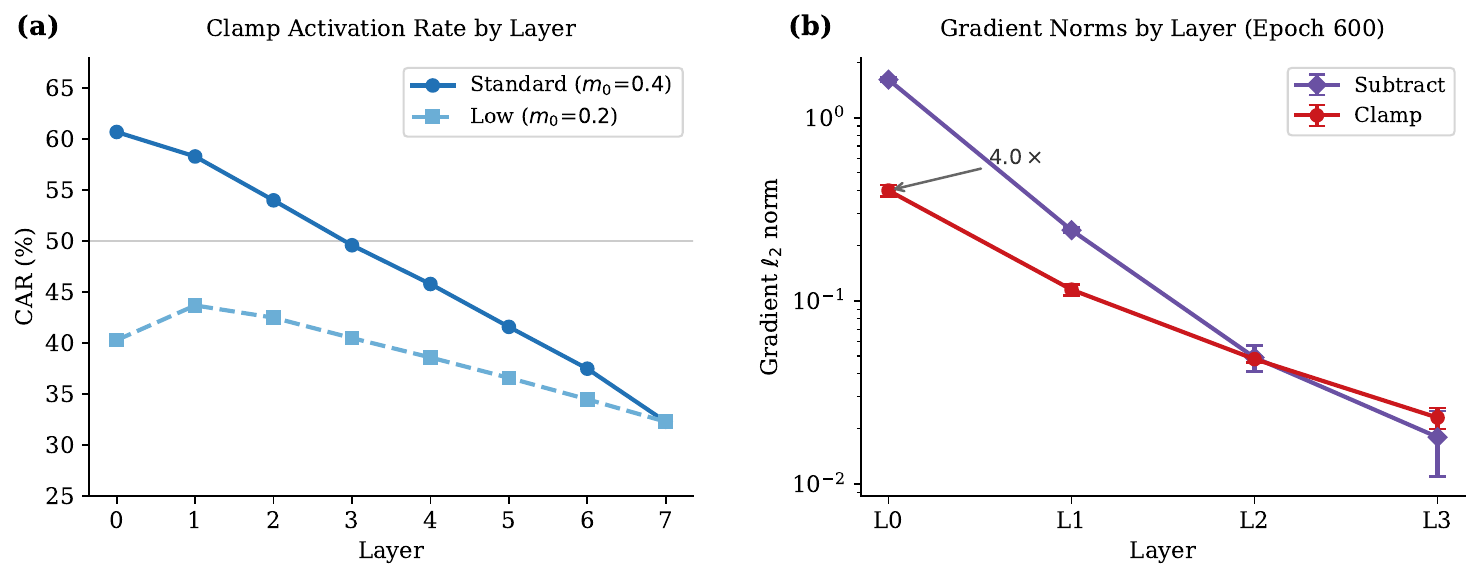}
\caption{\textbf{Diagnostic profiles by layer (CIFAR-10).} (a)~CAR
under standard and low-margin schedules. (b)~Gradient $\ell_2$ norms
at the final epoch. The gradient norm gap between conditions tracks
the CAR profile, closing by layer~2.}
\label{fig:car_gradnorms}
\end{figure}

\subsection{Mechanistic interpretation}
\label{sec:mech_interp}

The diagnostic evidence fits the following qualitative account. Under
the standard schedule, over 60\% of positive pairs at layer~0 exceed
the clamp threshold (Table~\ref{tab:car_full}); reducing the starting
margin halves this rate (Table~\ref{tab:car_full}, Reduction column).
This frequent saturation truncates gradients at early layers: gradient
norms at layer~0 are $4.0\times$ lower under clamping than under the
subtraction baseline (Table~\ref{tab:grad_norms}), and the reduction
tracks the CAR profile across layers
(Figure~\ref{fig:car_gradnorms}).

The truncation is not uniform: which positive pairs saturate depends on
the realized cosine similarities, which vary with initialization and
data ordering. Different seeds therefore experience different patterns
of gradient zeroing across minibatches, causing optimization
trajectories to separate more than under the non-truncating baseline.
The $5.90\times$ variance ratio (Table~\ref{tab:c10_pooled}) is the
aggregate outcome; halving the margin reduces VR to $2.98\times$
(Table~\ref{tab:dose_response}), consistent with less saturation
producing less separation.

The first two links in this chain (frequent saturation and reduced
gradient norms) are directly verified by our measurements. The
subsequent links, from stochastic truncation to trajectory separation
and from separation to accuracy variance, are inferred from the joint
pattern of evidence rather than demonstrated causally. In particular,
we do not track CAR or gradient norms over training (only at epoch
600), nor do we measure per-seed variation in CAR, which would more
directly test whether seeds with higher truncation rates produce
more extreme final accuracies.

% ============================================================================
% 7. CROSS-DATASET GENERALIZATION
% ============================================================================
\section{Cross-Dataset Generalization}
\label{sec:generalization}

The CIFAR-10 results establish that clamping can inflate variance when
CAR is high. To map out the conditions under which this effect appears,
we repeat the variance audit on three additional datasets that differ
in number of classes, positive-pair density, task difficulty, and visual
domain. Table~\ref{tab:cross_dataset_summary} summarizes the main
results; we discuss each dataset below.

\begin{table}[t]
\centering
\caption{\textbf{Cross-dataset summary.} Variance ratio (VR), layer-0
CAR, approximate positive pairs per minibatch (batch size 512), and
mean clamp accuracy. CIFAR-10 is the only dataset exhibiting
$\mathrm{VR} > 1$.}
\label{tab:cross_dataset_summary}
\begin{tabular}{lccccc}
\toprule
\textbf{Dataset} & \textbf{$K$} & \textbf{Pos.\ pairs}
  & \textbf{L0 CAR} & \textbf{Acc.\ (\%)} & \textbf{VR} \\
\midrule
CIFAR-10   & 10  & ${\sim}25{,}700$ & 60.7\% & 78.5 & $5.90\times$ \\
CIFAR-100  & 100 & ${\sim}2{,}580$  & 29.0\% & 51.8 & $0.39\times$ \\
SVHN       & 10  & ${\sim}30{,}500$ & 51.0\% & 96.7 & $0.25\times$ \\
Fash-MNIST & 10  & ${\sim}25{,}700$ & ---    & 92.6 & $0.08\times$ \\
\bottomrule
\end{tabular}
\end{table}

\subsection{CIFAR-100: Low positive-pair density}
\label{sec:cifar100}

With 100 classes and batch size 512, each minibatch contains roughly
$2{,}580$ positive pairs, about $10\times$ fewer than CIFAR-10's
${\sim}25{,}700$. We ran a $2 \times 2$ factorial ($n{=}5$ seeds per
cell; architecture and training identical to CIFAR-10 except
$K{=}100$).

\begin{table}[t]
\centering
\caption{\textbf{CIFAR-100 pooled results ($n{=}10$ per group).}}
\label{tab:c100_pooled}
\begin{tabular}{lcccc}
\toprule
\textbf{Condition} & \textbf{n} & \textbf{Mean (\%)}
  & \textbf{Std (\%)} & \textbf{Var} \\
\midrule
Clamp (pooled)    & 10 & 51.82 & 0.505 & 0.255 \\
Subtract (pooled) & 10 & 51.38 & 0.814 & 0.663 \\
\bottomrule
\end{tabular}
\end{table}

The variance ratio is $0.39\times$, inverted relative to CIFAR-10, and
non-significant ($F(9,9) = 0.39$, $p = 0.17$). CAR measurements on
CIFAR-100 reveal the reason: L0 CAR is only 29.0\%
(Table~\ref{tab:car_c100}), compared to 60.7\% on CIFAR-10, and
average CAR across layers is 19.7\% versus 47.5\%. With fewer positive
pairs per batch, fewer similarity values approach the saturation
boundary, so the gradient truncation pathway that drives variance
inflation on CIFAR-10 is largely inactive.

\begin{table}[t]
\centering
\caption{\textbf{CAR comparison: CIFAR-10 vs.\ CIFAR-100
(standard margin).}}
\label{tab:car_c100}
\begin{tabular}{cccc}
\toprule
\textbf{Layer} & \textbf{CAR (CIFAR-10)} & \textbf{CAR (CIFAR-100)}
  & \textbf{Difference} \\
\midrule
0 & 60.7\% & 29.0\% & $-31.7$ pp \\
1 & 58.3\% & 27.8\% & $-30.5$ pp \\
2 & 54.0\% & 24.2\% & $-29.8$ pp \\
3 & 49.6\% & 21.4\% & $-28.2$ pp \\
4 & 45.8\% & 18.4\% & $-27.4$ pp \\
5 & 41.6\% & 15.5\% & $-26.1$ pp \\
6 & 37.5\% & 12.3\% & $-25.2$ pp \\
7 & 32.3\% &  8.7\% & $-23.6$ pp \\
\midrule
\textbf{Mean} & \textbf{47.5\%} & \textbf{19.7\%} & $-27.8$ pp \\
\bottomrule
\end{tabular}
\end{table}

Layerwise gradient norms confirm that low CAR translates to an absence
of gradient truncation (Table~\ref{tab:gradnorms_c100}). At layer~0,
the clamp-to-subtract norm ratio is $0.58\times$: clamping produces
\emph{smaller} gradients, not larger ones. This is the reverse of
CIFAR-10, where the same ratio is $0.25\times$ (clamp norms are
$4.0\times$ lower, meaning subtract norms dominate). On CIFAR-100,
the two conditions produce similar gradient landscapes, consistent
with the absence of a variance effect.

\begin{table}[t]
\centering
\caption{\textbf{Mean gradient norms by layer (CIFAR-100, single seed).}
Contrast with CIFAR-10 (Table~\ref{tab:grad_norms}), where clamping
reduces L0 norms by $4.0\times$.}
\label{tab:gradnorms_c100}
\begin{tabular}{lcccc}
\toprule
\textbf{Layer} & \textbf{Margin} & \textbf{Clamp} & \textbf{Subtract}
  & \textbf{Ratio (C/S)} \\
\midrule
L0 & 0.400 & $1.265 \pm 0.099$ & $2.170 \pm 0.279$ & $0.58\times$ \\
L1 & 0.357 & $0.262 \pm 0.019$ & $0.358 \pm 0.032$ & $0.73\times$ \\
L2 & 0.314 & $0.080 \pm 0.004$ & $0.073 \pm 0.004$ & $1.09\times$ \\
L3 & 0.271 & $0.046 \pm 0.002$ & $0.029 \pm 0.002$ & $1.60\times$ \\
\bottomrule
\end{tabular}
\end{table}

\subsection{SVHN: High accuracy, high CAR}
\label{sec:svhn}

SVHN is a 10-class digit recognition dataset with non-uniform class
frequencies, yielding ${\sim}30{,}500$ positive pairs per minibatch.
We ran $n{=}5$ seeds each for \texttt{clamp\_direct} and
\texttt{subtract\_direct}.

\begin{table}[t]
\centering
\caption{\textbf{SVHN results ($n{=}5$ per group).}}
\label{tab:svhn_results}
\begin{tabular}{lcccc}
\toprule
\textbf{Condition} & \textbf{n} & \textbf{Mean (\%)}
  & \textbf{Std (\%)} & \textbf{Var} \\
\midrule
Clamp    & 5 & 96.70 & 0.430 & 0.185 \\
Subtract & 5 & 95.29 & 0.853 & 0.727 \\
\bottomrule
\end{tabular}
\end{table}

The variance ratio is $0.25\times$ (inverted; $F(4,4) = 0.25$,
$p = 0.21$). SVHN's CAR profile differs qualitatively from
CIFAR-10's: CAR \emph{increases} with depth (L0: 51.0\%, L7:
78.3\%), whereas CIFAR-10's decreases monotonically (L0: 60.7\%, L7:
32.3\%). Average CAR is 70.6\%, well above CIFAR-10's 47.5\%. Yet
the variance ratio is inverted.

This dissociation between CAR level and variance inflation shows that
high CAR is \emph{not sufficient} for the variance effect. We
hypothesize that SVHN's high accuracy (${\sim}97\%$) compresses the
room for seed-to-seed separation: when the task is nearly solved, all
seeds converge to similar optima regardless of gradient truncation.

\subsection{Fashion-MNIST: High accuracy, moderate difficulty}
\label{sec:fmnist}

Fashion-MNIST is a 10-class grayscale dataset. We ran $n{=}5$ seeds
each for clamp and subtract conditions.

\begin{table}[t]
\centering
\caption{\textbf{Fashion-MNIST results ($n{=}5$ per group).}}
\label{tab:fmnist_results}
\begin{tabular}{lcccc}
\toprule
\textbf{Condition} & \textbf{n} & \textbf{Mean (\%)}
  & \textbf{Std (\%)} & \textbf{Var} \\
\midrule
Clamp    & 5 & 92.57 & 0.113 & 0.0127 \\
Subtract & 5 & 91.87 & 0.406 & 0.1645 \\
\bottomrule
\end{tabular}
\end{table}

The variance ratio is $0.08\times$, the strongest inversion across all
datasets, and is significant ($F(4,4) = 0.08$, $p = 0.029$).
Fashion-MNIST's high accuracy (${\sim}92\%$) places it in the same
regime as SVHN: task difficulty is low enough that seeds converge to
similar optima despite gradient truncation. CAR was not measured for
Fashion-MNIST; based on the SVHN pattern (high accuracy co-occurring
with high CAR due to well-clustered representations), we would expect
moderate-to-high CAR that nonetheless does not produce variance
inflation.

\subsection{SVHN difficulty sweep}
\label{sec:svhn_sweep}

To test directly whether task difficulty moderates the
clamping-variance relationship, we repeated the SVHN experiment with
progressively more aggressive augmentation that lowers accuracy. This
holds the dataset, architecture, and number of classes constant while
varying only the optimization difficulty.

\begin{table}[t]
\centering
\caption{\textbf{SVHN difficulty sweep.} Increasing augmentation
difficulty on the same dataset drives VR from $0.25\times$ (easy,
${\sim}97\%$ accuracy) to $16.73\times$ (hard, ${\sim}25\%$ accuracy).}
\label{tab:svhn_sweep}
\begin{tabular}{lcccccl}
\toprule
\textbf{Difficulty} & \textbf{$n$} & \textbf{Acc.\ (\%)} & \textbf{Clamp var}
  & \textbf{Subtract var} & \textbf{VR} & \textbf{Pattern} \\
\midrule
Easy (standard) & 5 & 96.7 & 0.185 & 0.727 & $0.25\times$ & Inverted \\
Medium          & 5 & 26.8 & 37.57 & 17.23 & $2.18\times$ & Elevated \\
Hard            & 5 & 25.0 & 85.45 &  5.11 & $16.73\times$ & Extreme \\
\bottomrule
\end{tabular}
\end{table}

The sweep reveals a sharp transition: on the same dataset with the
same architecture, increasing augmentation difficulty moves the
variance ratio from $0.25\times$ to $16.73\times$. Under hard
augmentation, the clamp condition shows bimodal behavior, with some
seeds converging (${\sim}35\%$ accuracy) while others fail
(${\sim}18\%$), producing extreme variance. This is the strongest
evidence that task difficulty, not CAR alone, determines whether
clamping destabilizes training.

\subsection{Synthesis: When does clamping affect variance?}
\label{sec:synthesis}

The cross-dataset evidence suggests that clamping inflates variance
when two conditions co-occur:

\begin{enumerate}
    \item \textbf{High L0 clamp activation rate.} CIFAR-100's low
    L0 CAR (29.0\%) prevents the truncation pathway from activating.
    High CAR is necessary but not sufficient.

    \item \textbf{Intermediate task difficulty.} SVHN and
    Fashion-MNIST reach $> 92\%$ accuracy, leaving little room for
    trajectories to separate. At the other extreme, very low accuracy
    under hard augmentation produces extreme variance
    (Table~\ref{tab:svhn_sweep}).
\end{enumerate}

CIFAR-10 is the only dataset in our study that satisfies both
conditions: high L0 CAR (60.7\%) and intermediate accuracy
(${\sim}78\%$). The SVHN difficulty sweep confirms that these
conditions interact on a single dataset: the same architecture
produces $\mathrm{VR} = 0.25\times$ at 97\% accuracy and
$\mathrm{VR} = 16.73\times$ at 25\% accuracy.

% ============================================================================
% 8. DISCUSSION
% ============================================================================
\section{Discussion}
\label{sec:discussion}

\paragraph{Variance without mean cost.}
On CIFAR-10, the clamped variant produces $5.90\times$ higher
test-accuracy variance with no detectable difference in mean accuracy
($p = 0.92$; Table~\ref{tab:c10_pooled}). This means that switching
to the subtraction reference reduces noise at no observed cost to the
central estimate in this setting.

\paragraph{Practical consequences for seed budgets.}
The standard error of the sample mean scales as
$\mathrm{Std} / \sqrt{n}$. To put this concretely: under clamping on
CIFAR-10 ($\mathrm{Std} \approx 1.0$ pp), reaching a standard error
of $\pm 0.3$ pp requires $n \approx 11$ seeds; under subtraction
($\mathrm{Std} \approx 0.42$ pp), $n \approx 2$ is enough. This
difference can determine whether a hyperparameter comparison is
adequately powered under finite compute.

\paragraph{Why layer-local training may amplify the effect.}
In end-to-end backpropagation, gradient truncation at one layer is
modulated by the full chain of downstream gradients, and compensating
signals from other layers can partially offset local artifacts. In
CFF, each layer optimizes independently with no inter-layer gradient
flow, so saturation-induced truncation at layer~0 has no compensating
path from deeper layers. This structural difference may make CFF more
sensitive to margin-induced gradient truncation than end-to-end
training.

\paragraph{Dataset dependence and the role of positive-pair density.}
The inverted variance ratio on CIFAR-100 ($\mathrm{VR} = 0.39\times$)
is accounted for by its $10\times$ lower positive-pair density (100
classes vs.\ 10), which produces correspondingly lower CAR at every
layer (Table~\ref{tab:car_c100}). This confirms that the gradient
truncation mechanism depends on a sufficient density of positive pairs
to create widespread saturation. Layerwise gradient norms support this
picture: on CIFAR-100, the clamp-to-subtract gradient norm ratio at
layer~0 is $0.58\times$ (1.27 vs.\ 2.17), showing that clamping does
\emph{not} amplify gradient magnitudes when CAR is low. On CIFAR-10,
the analogous ratio is $4.0\times$ higher under clamping
(Table~\ref{tab:grad_norms}). Practitioners working with many-class
problems are unlikely to encounter the variance inflation documented
on CIFAR-10.

\paragraph{Task difficulty as a moderating factor.}
SVHN and Fashion-MNIST show that high CAR is not sufficient for
variance inflation: both reach $> 92\%$ accuracy, limiting
seed-to-seed spread regardless of truncation levels. The SVHN
difficulty sweep (Table~\ref{tab:svhn_sweep}) offers the most direct
evidence for this interaction, showing on a single dataset that
lowering accuracy from 97\% to 25\% moves the variance ratio from
$0.25\times$ to $16.73\times$. The clamping-variance effect thus
appears most pronounced in an intermediate-difficulty regime where
optimization trajectories are sensitive to gradient perturbations
but the task is not so easy that all seeds converge regardless.

\paragraph{Connection to margin saturation in metric learning.}
Hard margins in face recognition objectives (SphereFace, CosFace,
ArcFace) produce saturated logit regions by design. Our finding
suggests that variance audits, not just mean-accuracy comparisons,
may be informative when evaluating margin implementations in those
settings, though end-to-end gradient flow and different loss
reductions limit how far the analogy extends.

\paragraph{Practical guidelines.}
Based on the cross-dataset evidence:
(i)~In settings resembling CIFAR-10 (moderate accuracy, high
positive-pair density), switching to the subtraction reference yields
a large variance reduction at no observed mean cost.
(ii)~In high-accuracy settings ($> 90\%$) or many-class settings (low
positive-pair density), clamping does not appear to inflate variance
and may not need modification.
(iii)~When in doubt, measuring L0 CAR provides a simple diagnostic:
values well below 50\% suggest the truncation pathway is largely
inactive.

% ============================================================================
% 9. LIMITATIONS AND THREATS TO VALIDITY
% ============================================================================
\section{Limitations and Threats to Validity}
\label{sec:limitations}

\paragraph{Single architecture and training protocol.}
All results use a fixed ViT configuration ($d{=}128$, $H{=}4$,
$L{=}8$), training schedule (600 Stage-1 epochs), and augmentation
pipeline. We do not test how the variance effect depends on model
scale, depth, or alternative training schedules.

\paragraph{Cross-dataset design imbalances.}
The CIFAR-10 primary analysis uses the full $2 \times 2$ factorial
($n{=}7$ per cell, $n{=}14$ pooled). CIFAR-100 uses $n{=}5$ per cell
($n{=}10$ pooled). SVHN and Fashion-MNIST use $n{=}5$ per condition
without the stability-mode factor. These design differences reflect
compute allocation that prioritized the primary dataset; the
generalization experiments have correspondingly lower statistical
power.

\paragraph{SVHN design limitations.}
SVHN experiments used only the \texttt{direct} stability mode because
a bug in the \texttt{detach} condition was discovered after the fact.
The CIFAR-10 factorial establishes that stability mode does not affect
variance ($p > 0.82$), so this limitation is unlikely to change the
conclusions, but we note the asymmetry.

\paragraph{Missing Fashion-MNIST CAR.}
CAR was not measured for Fashion-MNIST due to checkpoint availability.
The significant variance inversion ($\mathrm{VR} = 0.08\times$,
$p = 0.029$) at 92.6\% accuracy is consistent with the
task-difficulty moderation pattern observed on SVHN, but without CAR
data we cannot confirm whether the truncation pathway is active or
inactive on this dataset.

\paragraph{Temperature-margin confound.}
Clamping (Eq.~\eqref{eq:clamp_margin}) is applied before temperature
scaling (Eq.~\eqref{eq:base_logits}), so its effective logit-space
strength is $m_\ell / \tau$ until saturation. In the dose-response
probe, we vary $m_0$ while holding $\tau = 0.15$ fixed, confounding
changes in saturation frequency with overall logit-scale changes.

\paragraph{SVHN difficulty sweep confounds.}
The difficulty sweep varies augmentation intensity, which simultaneously
changes task difficulty, the distribution of pairwise similarities, and
the effective positive-pair informativeness. We cannot separate which
of these factors drives the VR transition.

\paragraph{Augmentation sensitivity.}
Our augmentation pipeline uses 12-pixel padding for random crops on
$32 \times 32$ images, which is more aggressive than the standard
4-pixel padding. Because augmentation intensity affects the
distribution of pairwise cosine similarities, and hence the frequency
of clamp saturation, the size of the variance effect may differ under
milder augmentation. The within-experiment comparison remains valid
(all conditions share identical augmentation), but the absolute CAR
values and variance ratios should not be assumed to carry over to
other pipelines.

\paragraph{Two-stage pipeline.}
The reported test accuracy reflects both representation learning
(Stage~1) and a linear probe on frozen representations (Stage~2).
Our mechanistic analysis focuses entirely on Stage~1 (CAR, gradient
norms), but probe-stage variance could also contribute to the observed
seed-to-seed differences. We do not separate the two stages'
contributions to total variance.

\paragraph{Diagnostic timing.}
CAR and gradient norms are reported at the final Stage-1 epoch only.
We do not track how these quantities change over training.

\paragraph{Per-seed mechanistic link.}
CAR and gradient norms are reported as cross-seed averages. We do not
measure per-seed variation in these diagnostics, which would more
directly test whether seeds that experience more frequent truncation
produce more extreme final accuracies.

\paragraph{Statistical power.}
The standard-margin factorial uses $n = 7$ seeds per cell ($n = 14$
pooled per margin type). While the F-test rejects variance equality
at $p = 0.003$, the bootstrap CI for VR is wide ($[1.62, 15.80]$),
so the \emph{magnitude} of the effect is less precisely pinned down
than its \emph{direction}. The generalization datasets use even
smaller samples ($n = 5$), further limiting power for individual
dataset-level significance tests.

\paragraph{Reduction scope.}
Proposition~\ref{prop:logprob_margin_constant} establishes
gradient-neutrality under the mean-over-positives reduction used
throughout this paper. Under alternative reductions (e.g., sum over
positives), the gradient-neutrality result would not hold.

% ============================================================================
% 10. CONCLUSION
% ============================================================================
\section{Conclusion}
\label{sec:conclusion}

We presented a formal specification of the supervised contrastive
loss used in Contrastive Forward-Forward training and proved that a
common margin variant, post-log-probability subtraction, is
gradient-neutral under the mean-over-positives reduction
(Proposition~\ref{prop:logprob_margin_constant}), establishing it as
a true no-margin reference. Using this reference, we separated the
effect of the alternative implementation: saturating similarity
clamping via $\min(s + m, 1)$.

On CIFAR-10 ($2 \times 2$ factorial, $n{=}7$ seeds per cell),
clamping is associated with nearly six-fold higher pooled
test-accuracy variance ($F$-test $p = 0.003$) at statistically
indistinguishable mean accuracy. Three diagnostic analyses support
a saturation-mediated account: clamp activation rates exceed 60\% at
early layers, gradient norms at layer~0 are $4.0\times$ lower under
clamping, and halving the starting margin reduces the variance ratio
from $5.90\times$ to $2.98\times$.

This effect does not generalize uniformly. Replication on CIFAR-100,
SVHN, and Fashion-MNIST reveals inverted variance ratios in all three
cases. Cross-dataset analysis identifies two moderating factors:
positive-pair density (which determines CAR) and task difficulty
(which limits trajectory separation at high accuracy). An SVHN
difficulty sweep confirms this interaction on a single dataset, with
the variance ratio moving from $0.25\times$ at 97\% accuracy to
$16.73\times$ at 25\% accuracy.

The practical upshot is conditional: in the regime exemplified by
CIFAR-10 (moderate accuracy, many same-class pairs per batch),
switching to the gradient-neutral subtraction reference cuts training
variance at no observed cost. In high-accuracy or many-class settings,
the effect is absent or inverted, and clamping may not need
modification. We recommend that CFF practitioners measure L0 CAR as
a simple diagnostic and report the chosen margin implementation
explicitly to support reproducibility.

% ============================================================================
% BROADER IMPACT
% ============================================================================
\subsubsection*{Broader Impact Statement}
This work audits seed-to-seed variance arising from a margin
implementation choice in a contrastive loss. It introduces no new
model capabilities, datasets, or deployment-facing systems, and we
identify no specific pathway to negative societal impact.
% ============================================================================
% REFERENCES
% ============================================================================
\bibliographystyle{plainnat}
\bibliography{references}

% ============================================================================
% APPENDIX
% ============================================================================
\appendix

\section{Per-Seed Results}
\label{app:per_seed_c10}

Tables~\ref{tab:c10_per_seed_std} and~\ref{tab:c10_per_seed_low}
report individual test accuracies for all CIFAR-10 runs. Seeds are
numbered S1--S7 consistently within each condition; the same seed index
across conditions corresponds to the same random seed value. The
\texttt{subtract\_detach} rows are identical in both tables because
the subtraction baseline is gradient-neutral
(Proposition~\ref{prop:logprob_margin_constant}): the margin schedule
value does not affect the trained model.

\begin{table}[H]
\centering
\caption{\textbf{Per-seed CIFAR-10 results (standard margin
$0.4 \to 0.1$).}}
\label{tab:c10_per_seed_std}
\begin{tabular}{lccccccc}
\toprule
\textbf{Condition}
  & \textbf{S1} & \textbf{S2} & \textbf{S3}
  & \textbf{S4} & \textbf{S5} & \textbf{S6}
  & \textbf{S7} \\
\midrule
\texttt{clamp\_detach}
  & 78.49 & 78.27 & 78.43 & 76.96 & 78.81 & 78.59 & 80.12 \\
\texttt{clamp\_direct}
  & 78.07 & 77.91 & 79.42 & 76.23 & 79.24 & 78.79 & 79.45 \\
\texttt{subtract\_detach}
  & 78.82 & 79.29 & 78.87 & 79.02 & 77.87 & 78.37 & 78.84 \\
\texttt{subtract\_direct}
  & 78.27 & 78.31 & 78.68 & 78.23 & 78.37 & 78.32 & 77.92 \\
\bottomrule
\end{tabular}
\end{table}

\begin{table}[H]
\centering
\caption{\textbf{Per-seed CIFAR-10 results (low margin
$0.2 \to 0.1$).} The \texttt{subtract\_detach} row is
identical to Table~\ref{tab:c10_per_seed_std}; see text.}
\label{tab:c10_per_seed_low}
\begin{tabular}{lccccccc}
\toprule
\textbf{Condition}
  & \textbf{S1} & \textbf{S2} & \textbf{S3}
  & \textbf{S4} & \textbf{S5} & \textbf{S6}
  & \textbf{S7} \\
\midrule
\texttt{clamp\_detach}
  & 78.97 & 78.49 & 79.04 & 78.15 & 78.46 & 78.49 & 77.99 \\
\texttt{clamp\_direct}
  & 80.16 & 77.36 & 78.33 & 77.74 & 76.88 & 77.55 & 78.35 \\
\texttt{subtract\_detach}
  & 78.82 & 79.29 & 78.87 & 79.02 & 77.87 & 78.37 & 78.84 \\
\bottomrule
\end{tabular}
\end{table}

\section{SVHN CAR Profile}
\label{app:svhn_car}

\begin{table}[H]
\centering
\caption{\textbf{CAR comparison: CIFAR-10 vs.\ SVHN (standard margin).}
SVHN's CAR \emph{increases} with depth, opposite to CIFAR-10.}
\label{tab:car_svhn}
\begin{tabular}{cccc}
\toprule
\textbf{Layer} & \textbf{CAR (CIFAR-10)} & \textbf{CAR (SVHN)}
  & \textbf{Difference} \\
\midrule
0 & 60.7\% & 51.0\% & $-9.7$ pp \\
1 & 58.3\% & 65.3\% & $+7.0$ pp \\
2 & 54.0\% & 75.0\% & $+21.0$ pp \\
3 & 49.6\% & 78.3\% & $+28.7$ pp \\
\midrule
\textbf{Mean (L0--L7)} & \textbf{47.5\%} & \textbf{70.6\%}
  & $+23.1$ pp \\
\bottomrule
\end{tabular}
\end{table}

\section{Per-Seed Results: Generalization Datasets}
\label{app:per_seed_generalization}

Tables~\ref{tab:c100_per_seed}--\ref{tab:svhn_sweep_per_seed} report
individual test accuracies for all cross-dataset experiments. Seeds are
numbered S1--S5 consistently within each condition.

\begin{table}[H]
\centering
\caption{\textbf{Per-seed CIFAR-100 results (standard margin
$0.4 \to 0.1$, $n{=}5$ per cell).}}
\label{tab:c100_per_seed}
\begin{tabular}{lccccc}
\toprule
\textbf{Condition}
  & \textbf{S1} & \textbf{S2} & \textbf{S3}
  & \textbf{S4} & \textbf{S5} \\
\midrule
\texttt{clamp\_detach}
  & 51.98 & 52.06 & 52.32 & 51.62 & 52.36 \\
\texttt{clamp\_direct}
  & 50.84 & 51.60 & 51.62 & 52.43 & 51.36 \\
\texttt{subtract\_detach}
  & 50.85 & 52.25 & 52.31 & 50.17 & 50.90 \\
\texttt{subtract\_direct}
  & 50.78 & 52.52 & 51.98 & 50.81 & 51.23 \\
\bottomrule
\end{tabular}
\end{table}

\begin{table}[H]
\centering
\caption{\textbf{Per-seed SVHN results ($n{=}5$ per group,
\texttt{direct} mode only).}}
\label{tab:svhn_per_seed}
\begin{tabular}{lccccc}
\toprule
\textbf{Condition}
  & \textbf{S1} & \textbf{S2} & \textbf{S3}
  & \textbf{S4} & \textbf{S5} \\
\midrule
\texttt{clamp}
  & 97.16 & 96.42 & 97.11 & 96.17 & 96.66 \\
\texttt{subtract}
  & 95.09 & 96.11 & 95.36 & 93.95 & 95.92 \\
\bottomrule
\end{tabular}
\end{table}

\begin{table}[H]
\centering
\caption{\textbf{Per-seed Fashion-MNIST results ($n{=}5$ per group).}}
\label{tab:fmnist_per_seed}
\begin{tabular}{lccccc}
\toprule
\textbf{Condition}
  & \textbf{S1} & \textbf{S2} & \textbf{S3}
  & \textbf{S4} & \textbf{S5} \\
\midrule
\texttt{clamp}
  & 92.39 & 92.69 & 92.63 & 92.55 & 92.57 \\
\texttt{subtract}
  & 92.00 & 91.32 & 92.06 & 92.36 & 91.62 \\
\bottomrule
\end{tabular}
\end{table}

\begin{table}[H]
\centering
\caption{\textbf{Per-seed SVHN difficulty sweep results ($n{=}5$ per
condition).} Augmentation details in
Section~\ref{sec:exp_setup}.}
\label{tab:svhn_sweep_per_seed}
\begin{tabular}{llccccc}
\toprule
\textbf{Difficulty} & \textbf{Condition}
  & \textbf{S1} & \textbf{S2} & \textbf{S3}
  & \textbf{S4} & \textbf{S5} \\
\midrule
\multirow{2}{*}{Medium}
  & \texttt{clamp}    & 28.71 & 21.05 & 36.37 & 23.89 & 23.92 \\
  & \texttt{subtract} & 22.01 & 30.19 & 28.50 & 21.22 & 28.02 \\
\midrule
\multirow{2}{*}{Hard}
  & \texttt{clamp}    & 18.11 & 35.75 & 19.02 & 34.41 & 17.62 \\
  & \texttt{subtract} & 23.67 & 25.93 & 21.98 & 21.65 & 19.97 \\
\bottomrule
\end{tabular}
\end{table}

\section{Normality Tests}
\label{app:normality}

Table~\ref{tab:shapiro} reports Shapiro--Wilk tests for the pooled
groups used in the primary variance comparisons.

\begin{table}[H]
\centering
\caption{\textbf{Shapiro--Wilk normality tests for pooled groups.}}
\label{tab:shapiro}
\begin{tabular}{lccc}
\toprule
\textbf{Group} & \textbf{n} & \textbf{W} & \textbf{p-value} \\
\midrule
Clamp (standard margin)    & 14 & 0.9513 & 0.5816 \\
Subtract (standard margin) & 14 & 0.9516 & 0.5853 \\
Clamp (low margin)         & 14 & 0.9608 & 0.7355 \\
Subtract (low margin)      &  7 & 0.9074 & 0.3780 \\
\bottomrule
\end{tabular}
\end{table}

\section{Reproducibility Details}
\label{app:reproducibility}

\paragraph{Software.}
All experiments use PyTorch 2.6.0 with CUDA 12.4 (system driver
CUDA 12.8) on NVIDIA V100-SXM2-32GB GPUs.
The CFF implementation is based on the codebase of
\citet{aghagolzadeh2025cff} with modifications to support the
subtraction margin variant and diagnostic logging.

\paragraph{Seed values.}
CIFAR-10 experiments use seven random seeds: 1, 2, 3, 4, 5, 6, 7.
Cross-dataset experiments (CIFAR-100, SVHN, Fashion-MNIST) use seeds
1--5. Each seed is passed to \texttt{torch.manual\_seed},
\texttt{numpy.random.seed}, and \texttt{random.seed}, and
\texttt{torch.cuda.manual\_seed\_all} for full determinism.

\paragraph{Bootstrap procedure.}
All bootstrap confidence intervals use 10{,}000 percentile resamples.
For the variance ratio, each resample draws $n$ values with
replacement from each group independently, computes the sample
variance of each resampled group, and takes their ratio. The reported
95\% CI is the 2.5th and 97.5th percentiles of the resulting
distribution.

\paragraph{Diagnostic computation.}
CAR values (Eq.~\eqref{eq:car}) are computed at the final Stage-1
epoch (epoch 600) by evaluating Eq.~\eqref{eq:car} on each
minibatch of the training set and averaging across minibatches, then
averaging across seeds. Gradient norms (Eq.~\eqref{eq:grad_norm})
are computed on a single minibatch at the final epoch and averaged
across seeds.

\end{document}